\documentclass{article}
\usepackage{kr}

\usepackage{times}
\usepackage{soul}
\usepackage{url}
\usepackage[hidelinks]{hyperref}
\usepackage[utf8]{inputenc}
\usepackage[small]{caption}
\usepackage{graphicx}
\usepackage{amsmath}
\usepackage{amsthm}
\usepackage{booktabs}
\usepackage{algorithm}
\usepackage{algorithmic}
\usepackage{latexsym}

\usepackage{amssymb}
\usepackage{mathtools}
\usepackage{stmaryrd}
\usepackage{esvect}
\usepackage{braket}

\usepackage[dvipsnames]{xcolor}
\usepackage{xstring}
\usepackage{xspace}
\usepackage{graphicx}
\usepackage{boxedminipage}
\usepackage{multirow}
\usepackage{url}
\usepackage{enumerate}
\usepackage[normalem]{ulem}

\usepackage{subcaption}
\usepackage{tikz}

\usepackage[kerning=true,tracking=true, babel=true,spacing=true]{microtype} 

\newtheorem{example}{Example}

\newtheorem{definition}{Definition}
\newtheorem{proposition}{Proposition}

\newcommand{\ol}[1]{\overline{#1}}

\renewcommand{\leq}{\leqslant}

\newcommand{\cspMIof}[2]{\ensuremath{\mathit{CR}^{#2}(#1)}} %

\newcommand{\OCFsolutionsOf}[1]{\ensuremath{\mathit{Sol_{OCF}}(#1)}}

\newcommand{\OCFsolutionsRnMI}[1]{\ensuremath{\OCFsolutionsOf{\cspMIof{\R_n}{n-1}}}}

\newcommand{\syntheticKB}[1]{\texttt{kb\_synth<\(n\)>\_c<\(2n\!\!-\!\!1\)>.pl}}

\newcommand{\R}{\ensuremath{\mathcal R}}

\newcommand{\Mod}{\mbox{\it Mod}\,}

\makeatletter
\ifx\centernot\undefined
\newcommand*{\centernot}{%
	\mathpalette\@centernot
}
\fi
\def\@centernot#1#2{%
	\mathrel{%
		\rlap{%
			\settowidth\dimen@{$\m@th#1{#2}$}%
			\kern.5\dimen@
			\settowidth\dimen@{$\m@th#1=$}%
			\kern-.5\dimen@
			$\m@th#1\not$%
		}%
		{#2}%
	}%
}
\makeatother
\DeclareRobustCommand\nmableitSymb{\mathrel|\mkern-.5mu\joinrel\sim} %
\newcommand{\nmableit}{\ensuremath{\mbox{$\,\nmableitSymb\,$}}} %

\newcommand{\beweisendezeichen}%
{\penalty50\hspace*{0pt plus 1fil}\parfillskip=0pt\mbox{$\Box$}}

\newcommand{\fussnoteOhneMarkierung}[1]%
{%
\footnote{#1}%
\addtocounter{footnote}{-1}%
}

\newcommand{\satzCL}[2]{\ensuremath{(#1|#2)}}

\newlength{\abstand}
\newcommand{\LpropSig}[1]{\ensuremath{\mathcal{L}_{#1}}}

\newcommand{\LcondSig}[1]{\ensuremath{\satzCL{\mathcal{L}}{\mathcal{L}}_{#1}}}

\newcommand{\UnivSig}[1]{\ensuremath{\Omega_{#1}}}

\newcommand{\thistheoremname}{}

\newcommand\worldMargin[2]{{#1}_{\mid {#2}}}

\newcommand{\wableit}[1]{\ensuremath{\nmableit_{#1}^\textsf{w}}}
\newcommand{\wableitDelta}{\wableit{\Delta}}
\newcommand{\worder}[1]{\ensuremath{<_{#1}^\textsf{w}}}
\newcommand{\worderDelta}{\worder{\Delta}}
\newcommand{\wordereq}[1]{\ensuremath{\leq_{#1}^\textsf{w}}}
\newcommand{\wordereqDelta}{\wordereq{\Delta}}

\newcommand{\signa}{\Sigma}

\newcommand{\nmDelta}{{\,|\hspace{-0.5em}\sim_{\Delta}\,}}

\newcommand{\nmany}[2]{{\,|\hspace{-0.5em}\sim^{#2}_{#1}\,}}
\newcommand{\splitpost}{(SynSplit)}

\newcommand{\splitcup}{\bigcup\limits_{\signa_1, \signa_2}}
\newcommand{\relevance}{(Rel)}

\newcommand{\indep}{(Ind)}

\newcommand{\fctOPname}{{\ensuremath{\mathit{OP}}}} %
\newcommand{\fctOP}[1]{\ensuremath{\fctOPname(#1)}} %

\newcommand{\fctOPvector}[1]{\ensuremath{(\R_0,\ldots,\R_k)}} %

\newcommand{\falsWname}{\ensuremath{\mathit{\xi}}}
\newcommand{\falsW}{\ensuremath{\falsWname}}

\newcommand{\nameSZ}{W\xspace}
\newcommand{\systemSZ}{{system~\nameSZ}\xspace}

\newcommand{\namePS}{preferred structure\xspace}

\newcommand{\perpendicular}%
{\begin{sideways}$\models$\end{sideways}}

\newlength{\spalteAbst}
\newlength{\spalteAbstGr}
\newlength{\spalteAbstGGr}
\newlength{\zeileAbst}
\newcommand{\ALvariableE}{\ensuremath{a}}
\newcommand{\ALvariableZ}{\ensuremath{b}}
\newcommand{\ALvariableD}{\ensuremath{c}}
\newcommand{\ALvariableV}{\ensuremath{d}}
\newcommand{\ALvariableF}{\ensuremath{e}}

\newlength{\AbstandZwischenLiteralen}
\newcommand{\setzeAbstandZwischenLiteralen}{\hspace{\AbstandZwischenLiteralen}}

\newcommand{\omegaFuenf}[1]{%
    \IfEqCase{#1}{%
        {01}{\ensuremath{\omegaAafuenf{\ALvariableE}{\ALvariableZ}{\ALvariableD}{\ALvariableV}{\ALvariableF}}}%
        {02}{\ensuremath{\omegaAbfuenf{\ALvariableE}{\ALvariableZ}{\ALvariableD}{\ALvariableV}{\ALvariableF}}}%
        {03}{\ensuremath{\omegaBafuenf{\ALvariableE}{\ALvariableZ}{\ALvariableD}{\ALvariableV}{\ALvariableF}}}%
        {04}{\ensuremath{\omegaBbfuenf{\ALvariableE}{\ALvariableZ}{\ALvariableD}{\ALvariableV}{\ALvariableF}}}%
        {05}{\ensuremath{\omegaCafuenf{\ALvariableE}{\ALvariableZ}{\ALvariableD}{\ALvariableV}{\ALvariableF}}}%
        {06}{\ensuremath{\omegaCbfuenf{\ALvariableE}{\ALvariableZ}{\ALvariableD}{\ALvariableV}{\ALvariableF}}}%
        {07}{\ensuremath{\omegaDafuenf{\ALvariableE}{\ALvariableZ}{\ALvariableD}{\ALvariableV}{\ALvariableF}}}%
        {08}{\ensuremath{\omegaDbfuenf{\ALvariableE}{\ALvariableZ}{\ALvariableD}{\ALvariableV}{\ALvariableF}}}%
        {09}{\ensuremath{\omegaEafuenf{\ALvariableE}{\ALvariableZ}{\ALvariableD}{\ALvariableV}{\ALvariableF}}}%
        {10}{\ensuremath{\omegaEbfuenf{\ALvariableE}{\ALvariableZ}{\ALvariableD}{\ALvariableV}{\ALvariableF}}}%
        {11}{\ensuremath{\omegaFafuenf{\ALvariableE}{\ALvariableZ}{\ALvariableD}{\ALvariableV}{\ALvariableF}}}%
        {12}{\ensuremath{\omegaFbfuenf{\ALvariableE}{\ALvariableZ}{\ALvariableD}{\ALvariableV}{\ALvariableF}}}%
        {13}{\ensuremath{\omegaGafuenf{\ALvariableE}{\ALvariableZ}{\ALvariableD}{\ALvariableV}{\ALvariableF}}}%
        {14}{\ensuremath{\omegaGbfuenf{\ALvariableE}{\ALvariableZ}{\ALvariableD}{\ALvariableV}{\ALvariableF}}}%
        {15}{\ensuremath{\omegaHafuenf{\ALvariableE}{\ALvariableZ}{\ALvariableD}{\ALvariableV}{\ALvariableF}}}%
        {16}{\ensuremath{\omegaHbfuenf{\ALvariableE}{\ALvariableZ}{\ALvariableD}{\ALvariableV}{\ALvariableF}}}%
        {17}{\ensuremath{\omegaIafuenf{\ALvariableE}{\ALvariableZ}{\ALvariableD}{\ALvariableV}{\ALvariableF}}}%
        {18}{\ensuremath{\omegaIbfuenf{\ALvariableE}{\ALvariableZ}{\ALvariableD}{\ALvariableV}{\ALvariableF}}}%
        {19}{\ensuremath{\omegaJafuenf{\ALvariableE}{\ALvariableZ}{\ALvariableD}{\ALvariableV}{\ALvariableF}}}%
        {20}{\ensuremath{\omegaJbfuenf{\ALvariableE}{\ALvariableZ}{\ALvariableD}{\ALvariableV}{\ALvariableF}}}%
        {21}{\ensuremath{\omegaKafuenf{\ALvariableE}{\ALvariableZ}{\ALvariableD}{\ALvariableV}{\ALvariableF}}}%
        {22}{\ensuremath{\omegaKbfuenf{\ALvariableE}{\ALvariableZ}{\ALvariableD}{\ALvariableV}{\ALvariableF}}}%
        {23}{\ensuremath{\omegaLafuenf{\ALvariableE}{\ALvariableZ}{\ALvariableD}{\ALvariableV}{\ALvariableF}}}%
        {24}{\ensuremath{\omegaLbfuenf{\ALvariableE}{\ALvariableZ}{\ALvariableD}{\ALvariableV}{\ALvariableF}}}%
        {25}{\ensuremath{\omegaMafuenf{\ALvariableE}{\ALvariableZ}{\ALvariableD}{\ALvariableV}{\ALvariableF}}}%
        {26}{\ensuremath{\omegaMbfuenf{\ALvariableE}{\ALvariableZ}{\ALvariableD}{\ALvariableV}{\ALvariableF}}}%
        {27}{\ensuremath{\omegaNafuenf{\ALvariableE}{\ALvariableZ}{\ALvariableD}{\ALvariableV}{\ALvariableF}}}%
        {28}{\ensuremath{\omegaNbfuenf{\ALvariableE}{\ALvariableZ}{\ALvariableD}{\ALvariableV}{\ALvariableF}}}%
        {29}{\ensuremath{\omegaOafuenf{\ALvariableE}{\ALvariableZ}{\ALvariableD}{\ALvariableV}{\ALvariableF}}}%
        {30}{\ensuremath{\omegaObfuenf{\ALvariableE}{\ALvariableZ}{\ALvariableD}{\ALvariableV}{\ALvariableF}}}%
        {31}{\ensuremath{\omegaPafuenf{\ALvariableE}{\ALvariableZ}{\ALvariableD}{\ALvariableV}{\ALvariableF}}}%
        {32}{\ensuremath{\omegaPbfuenf{\ALvariableE}{\ALvariableZ}{\ALvariableD}{\ALvariableV}{\ALvariableF}}}%
    }[\PackageError{omegaFuenf}{Undefined option to omegaFuenf: #1}{}]%
}

\newcommand{\druckeWeltFuenf}[5]{\ensuremath{#1\setzeAbstandZwischenLiteralen #2\setzeAbstandZwischenLiteralen #3\setzeAbstandZwischenLiteralen #4\setzeAbstandZwischenLiteralen #5}} %

\newcommand{\omegaAafuenf}[5]{\ensuremath{\druckeWeltFuenf{#1}{#2}{#3}{#4}{#5}}}
\newcommand{\omegaAbfuenf}[5]{\ensuremath{\druckeWeltFuenf{#1}{#2}{#3}{#4}{\ol{#5}}}}
\newcommand{\omegaBafuenf}[5]{\ensuremath{\druckeWeltFuenf{#1}{#2}{#3}{\ol{#4}}{#5}}}
\newcommand{\omegaBbfuenf}[5]{\ensuremath{\druckeWeltFuenf{#1}{#2}{#3}{\ol{#4}}{\ol{#5}}}}
\newcommand{\omegaCafuenf}[5]{\ensuremath{\druckeWeltFuenf{#1}{#2}{\ol{#3}}{#4}{#5}}}
\newcommand{\omegaCbfuenf}[5]{\ensuremath{\druckeWeltFuenf{#1}{#2}{\ol{#3}}{#4}{\ol{#5}}}}
\newcommand{\omegaDafuenf}[5]{\ensuremath{\druckeWeltFuenf{#1}{#2}{\ol{#3}}{\ol{#4}}{#5}}}
\newcommand{\omegaDbfuenf}[5]{\ensuremath{\druckeWeltFuenf{#1}{#2}{\ol{#3}}{\ol{#4}}{\ol{#5}}}}

\newcommand{\omegaEafuenf}[5]{\ensuremath{\druckeWeltFuenf{#1}{\ol{#2}}{#3}{#4}{#5}}}
\newcommand{\omegaEbfuenf}[5]{\ensuremath{\druckeWeltFuenf{#1}{\ol{#2}}{#3}{#4}{\ol{#5}}}}
\newcommand{\omegaFafuenf}[5]{\ensuremath{\druckeWeltFuenf{#1}{\ol{#2}}{#3}{\ol{#4}}{#5}}}
\newcommand{\omegaFbfuenf}[5]{\ensuremath{\druckeWeltFuenf{#1}{\ol{#2}}{#3}{\ol{#4}}{\ol{#5}}}}
\newcommand{\omegaGafuenf}[5]{\ensuremath{\druckeWeltFuenf{#1}{\ol{#2}}{\ol{#3}}{#4}{#5}}}
\newcommand{\omegaGbfuenf}[5]{\ensuremath{\druckeWeltFuenf{#1}{\ol{#2}}{\ol{#3}}{#4}{\ol{#5}}}}
\newcommand{\omegaHafuenf}[5]{\ensuremath{\druckeWeltFuenf{#1}{\ol{#2}}{\ol{#3}}{\ol{#4}}{#5}}}
\newcommand{\omegaHbfuenf}[5]{\ensuremath{\druckeWeltFuenf{#1}{\ol{#2}}{\ol{#3}}{\ol{#4}}{\ol{#5}}}}

\newcommand{\omegaIafuenf}[5]{\ensuremath{\druckeWeltFuenf{\ol{#1}}{#2}{#3}{#4}{#5}}}
\newcommand{\omegaIbfuenf}[5]{\ensuremath{\druckeWeltFuenf{\ol{#1}}{#2}{#3}{#4}{\ol{#5}}}}
\newcommand{\omegaJafuenf}[5]{\ensuremath{\druckeWeltFuenf{\ol{#1}}{#2}{#3}{\ol{#4}}{#5}}}
\newcommand{\omegaJbfuenf}[5]{\ensuremath{\druckeWeltFuenf{\ol{#1}}{#2}{#3}{\ol{#4}}{\ol{#5}}}}
\newcommand{\omegaKafuenf}[5]{\ensuremath{\druckeWeltFuenf{\ol{#1}}{#2}{\ol{#3}}{#4}{#5}}}
\newcommand{\omegaKbfuenf}[5]{\ensuremath{\druckeWeltFuenf{\ol{#1}}{#2}{\ol{#3}}{#4}{\ol{#5}}}}
\newcommand{\omegaLafuenf}[5]{\ensuremath{\druckeWeltFuenf{\ol{#1}}{#2}{\ol{#3}}{\ol{#4}}{#5}}}
\newcommand{\omegaLbfuenf}[5]{\ensuremath{\druckeWeltFuenf{\ol{#1}}{#2}{\ol{#3}}{\ol{#4}}{\ol{#5}}}}

\newcommand{\omegaMafuenf}[5]{\ensuremath{\druckeWeltFuenf{\ol{#1}}{\ol{#2}}{#3}{#4}{#5}}}
\newcommand{\omegaMbfuenf}[5]{\ensuremath{\druckeWeltFuenf{\ol{#1}}{\ol{#2}}{#3}{#4}{\ol{#5}}}}
\newcommand{\omegaNafuenf}[5]{\ensuremath{\druckeWeltFuenf{\ol{#1}}{\ol{#2}}{#3}{\ol{#4}}{#5}}}
\newcommand{\omegaNbfuenf}[5]{\ensuremath{\druckeWeltFuenf{\ol{#1}}{\ol{#2}}{#3}{\ol{#4}}{\ol{#5}}}}
\newcommand{\omegaOafuenf}[5]{\ensuremath{\druckeWeltFuenf{\ol{#1}}{\ol{#2}}{\ol{#3}}{#4}{#5}}}
\newcommand{\omegaObfuenf}[5]{\ensuremath{\druckeWeltFuenf{\ol{#1}}{\ol{#2}}{\ol{#3}}{#4}{\ol{#5}}}}
\newcommand{\omegaPafuenf}[5]{\ensuremath{\druckeWeltFuenf{\ol{#1}}{\ol{#2}}{\ol{#3}}{\ol{#4}}{#5}}}
\newcommand{\omegaPbfuenf}[5]{\ensuremath{\druckeWeltFuenf{\ol{#1}}{\ol{#2}}{\ol{#3}}{\ol{#4}}{\ol{#5}}}}

\newtheorem*{genericpostulate}{\thistheoremname}

\newtheorem{lemma}{Lemma}

\title{Inference with System W Satisfies Syntax Splitting}

\author{%
Jonas Haldimann\and
Christoph Beierle %
\affiliations
FernUniversität in Hagen, 58084 Hagen, Germany\\
\emails
\{jonas.haldimann, christoph.beierle\}@fernuni-hagen.de
}

\begin{document}

\maketitle

\begin{abstract}
	In this paper, we investigate inductive inference with system W from conditional belief bases with respect to syntax splitting.
The concept of syntax splitting for inductive inference states that inferences about independent parts of the signature should not affect each other.
This was captured in work by Kern-Isberner, Beierle, and Brewka in the form of postulates for inductive inference operators expressing syntax splitting as a combination of relevance and independence; it was also shown that c-inference fulfils syntax splitting, while system P inference and system Z both fail to satisfy it.
  System W is a recently introduced inference system for nonmonotonic reasoning that captures and properly extends system Z %
  as well as c-inference.
We show that system W fulfils the syntax splitting postulates for inductive inference operators by showing that it satisfies the required properties of relevance and independence.
This makes system W another inference operator besides c-inference that fully complies with syntax splitting, while in contrast to c-inference, also extending rational closure.
\end{abstract}

\section{Introduction}
\label{sec_introduction}

An important subject in the field of knowledge representation and reasoning is the reasoning with conditional knowledge \cite{LehmannMagidor92short}.
A conditional  formalizes a defeasible rule ``If \(A\) then usually \(B\)'' for logical formulas \(A, B\), in the following denoted as \(\satzCL{B}{A}\).
Two well known inference
methods for
conditional belief bases consisting of such conditionals
are p-entailment
that is characterized by the axioms of System P \cite{Adams1965,KrausLehmannMagidor1990} 
and system Z \cite{Pearl1990systemZTARK,GoldszmidtPearl1996}.
Newer approaches include
inference with c-representations \cite{KernIsberner2001,KernIsberner2004AMAI}, 
skeptical c-inference
taking all c-representations
into account
\cite{BeierleEichhornKernIsbernerKutsch2018AMAI}, and
the recently introduced system W \cite{KomoBeierle2022AMAI}.

While all reasoning approaches cited above satisfy the axioms of system P,
called the ``industry standard'' for qualitative nonmonotonic inference
\cite{HawthorneMakinson2007StudiaLogica},
there are differences among them with respect to other properties.
This also applies to
the highly desirable property of syntax splitting for nonmonotonic reasoning.
The concept of syntax splitting was originally developed by Parikh \cite{Parikh99} for belief sets in order 
to formulate a postulate for belief revision stating that
the revision with a formula that contains only variables from one of part of the signature should only affect the information about that part of the signature.
The notion of syntax splitting was later extended, e.g. \cite{PeppasWilliamsChopraFoo2015,Kern-IsbernerBrewka17}.
In \cite{KernIsbernerBeierleBrewka2020KR}, syntax splitting is introduced for nonmonotic resoning as a combination of \emph{relevance} and \emph{indpendence}, stating that
only conditionals from a considered part of the syntax splitting of  a belief base are relevant for corresponding inferences, and that inferences using only atoms from one part of the syntax splitting should be independent of the other parts.
It is shown that c-inference
fulfils syntax splitting, while
system P and system Z both fail to satisfy it
\cite{KernIsbernerBeierleBrewka2020KR}.

System W has been shown \cite{KomoBeierle2022AMAI} to exhibit  high-quality properties like
capturing and properly extending 
p-entailment,
system Z,
and c-inference,
or avoiding the drowning problem
\cite{Pearl1990systemZTARK,BenferhatCayrolDuboisLangPrade1993}.
In this paper, we show that system W also satisfies the required properties of relevance and
independence, making it another inference operator, besides c-inference, to fully
comply with the highly desirable property of syntax splitting. Furthermore, %
system W %
also extends,
in contrast to c-inference, rational closure
and thus inheriting its desirable properties \cite{LehmannMagidor92short}.

After briefly recalling the needed basics of conditional logic in Sec.~\ref{sec_background},
the syntax splitting postulates are given in Sec.~\ref{sec_synsplit}.
In Sec.~\ref{sec:systemW_preferred_structure}, we present an syntax splitting example
and     illustrate how system W handles it,
and Sec.~\ref{sec:systemW_splitting} shows that system W satisfies syntax splitting.
Sec.~\ref{sec:conclusion} concludes and points out further work.

\section{Reasoning with Conditional Logic}
\label{sec_background}

A \emph{(propositional) signature} is a finite set \(\Sigma\) of identifiers.
For a signature \(\Sigma\), we denote the propositional language over \(\Sigma\) by \(\LpropSig{\Sigma}\).
Usually, we denote elements of the signatures with lowercase letters \(a, b, c, \dots\) and formulas with uppercase letters \(A, B, C, \dots\).
We may denote a conjunction \(A \wedge B\) by \(AB\) and a negation \(\neg A\) by \(\ol{A}\) for brevity of notation.
The set of interpretations over  a signature \(\Sigma\) is denoted as \(\UnivSig{\Sigma}\).
Interpretations are also called \emph{worlds}. %
An interpretation \(\omega \in \UnivSig{\Sigma}\) is a \emph{model} of a formula \(A \in \LpropSig{\Sigma}\) if \(A\) holds in \(\omega\). This is denoted as \(\omega \models A\).
The set of models of a formula (over a signature \(\Sigma\)) is denoted as \(\Mod_\Sigma(A) = \{\omega \in \UnivSig{\Sigma} \mid \omega \models A\}\).
A formula \(A\) \emph{entails} a formula \(B\) if \(\Mod_\Sigma(A) \subseteq \Mod_\Sigma(B)\).

Worlds over (sub-)signatures can be merged or marginalized.
	Let \(\Sigma\) be a signature with disjunct sub-signatures \(\Sigma_1, \Sigma_2\) such that \(\Sigma = \Sigma_1 \cup \Sigma_2\).
	Let \(\omega_1 \in \Omega_{\Sigma_1}\) and \(\omega_2 \in \Omega_{\Sigma_2}\).
	Then
	\(
	(\omega_1 \cdot \omega_2)
	\)
	denotes the world from \(\Omega_\Sigma\) that assigns the truth values for variables in \(\Sigma_1\) as \(\omega_1\) and truth values for variables in \(\Sigma_2\) as \(\omega_2\).
	For \(\omega \in \Omega_\Sigma\), the world from \(\Omega_{\Sigma_1}\) that assigns the truth values for variables in \(\Sigma_1\) as \(\omega\) is denoted as \(\worldMargin{\omega}{\Sigma_1}\).

A \emph{conditional} \(\satzCL{B}{A}\) connects two formulas \(A, B\) and represents the rule ``If \(A\) then usually \(B\)''.
The conditional language over a signature \(\Sigma\) is denoted as \(\LcondSig{\Sigma} = \{\satzCL{B}{A} \mid A, B \in \LpropSig{\Sigma}\}\).
A finite set of conditionals is called a \emph{(conditional) belief base} \(\Delta\).

We use a three-valued semantics of conditionals in this paper \cite{deFinetti37orig}.
For a world \(\omega\) a conditional \(\satzCL{B}{A}\) is either \emph{verified} by \(\omega\) if \(\omega \models AB\), \emph{falsified} by \(\omega\) if \(\omega \models A\ol{B}\), or \emph{not applicable} to \(\omega\) if \(\omega \models \ol{A}\).

Reasoning with conditionals is often modelled by inference relations.
An \emph{inference relation} is a binary relation \(\nmableit\) on formulas over an underlying signature \(\Sigma\) with the intuition that \(A \nmableit B\) means that \(A\) (plausibly) entails \(B\).
(Non-monotonic) inference is closely related to conditionals: an inference relation \(\nmableit\) can also be seen as a set of conditionals \(\{\satzCL{B}{A} \mid A, B \in \LpropSig{\Sigma}, A \nmableit B\}\).

The following definition formalizes the inductive completion of a belief base according to an inference method.
\begin{definition}[inductive inference operator \cite{KernIsbernerBeierleBrewka2020KR}]
An \emph{inductive inference operator} is a mapping \(C: \Delta \mapsto \nmDelta\) that maps each belief base to an inference relation such that direct inference  (DI) and trivial vacuity (TV) are fulfilled, i.e.,
\begin{description}
	\item[(DI)] if \(\satzCL{B}{A} \in \Delta\) then \(A \nmDelta B\) and
	\item[(TV)] if \(\Delta = \emptyset\) and \(A \nmDelta B\) then \(A \models B\).
\end{description}
\end{definition}

Examples for inductive inference operators are p-entailment \cite{Adams1965} and system Z \cite{Pearl1990systemZTARK}.
\section{Syntax Splitting for Inductive Inference}
\label{sec_synsplit}

First, we recall the notion of syntax splitting for belief bases. %

\begin{definition}[syntax splitting for belief bases (adapted from \cite{KernIsbernerBeierleBrewka2020KR})]
	Let \(\Delta\) be a belief base over a signature \(\Sigma\).
	A partitioning \(\{\Sigma_1, \dots, \Sigma_n\}\) of \(\Sigma\) is a \emph{syntax splitting} for \(\Delta\) if there is a partitioning \(\{\Delta_1, \dots, \Delta_n\}\) of \(\Delta\) such that \(\Delta_i \subseteq \LcondSig{\Sigma_i}\) for every \(i = 1, \dots, n\).
	A syntax splitting \(\{\Sigma_1, \Sigma_2\}\) of \(\Delta\) with two parts and corresponding partition \(\{\Delta_1, \Delta_2\}\) of \(\Delta\) is denoted as
	\[
	\Delta = \Delta_1 \splitcup \Delta_2.
	\]
\end{definition}

Here, we will focus on syntax splittings in two sub-signatures.
Results for belief bases with syntax splittings in more than two parts can be obtained by iteratively applying the postulates presented here.

For belief bases with syntax splitting,
the postulate \relevance\ describes that conditionals corresponding to one part of the syntax splitting do not have any influence on inferences that only use the other part of the syntax splitting, i.e., that only conditionals from the considered part of the syntax splitting are relevant.

\begin{description}
	\item[\textbf{\relevance}] An  inductive inference operator \(C: \Delta \mapsto \nmDelta\) satisfies \textbf{\relevance} \cite{KernIsbernerBeierleBrewka2020KR} if for any \(\Delta = \Delta_1 \splitcup \Delta_2\),
	and for any \(A, B \in \LpropSig{\Sigma_i}\) for \(i =1, 2\) we have that
	\begin{equation}
		\label{eq_rel_qual}
		A \nmDelta B \quad \mbox{iff} \quad A \nmany{\Delta_i}{} B.
	\end{equation}
\end{description}

The postulate \indep\ describes that inferences should not be affected by beliefs in formulas over other sub-signatures in the splitting,
i.e., inferences using only atoms from one part of the syntax splitting should be drawn independently of beliefs about other parts of the splitting.

\begin{description}
	\item[\textbf{\indep}]
	An  inference operator \(C: \Delta \mapsto \nmDelta\) satisfies \textbf{\indep} \cite{KernIsbernerBeierleBrewka2020KR} if for any $\Delta = \Delta_1 \splitcup \Delta_2$, and for any \(A, B \in \LpropSig{\Sigma_i}\), \(D \in \LpropSig{\Sigma_j}\) for 
	\(i, j \in \{1, 2\}, i \neq j\)
	such that \(D\) is consistent, we have
	\begin{equation}
		\label{eq_ind_qual}
		A \nmDelta B \quad \mbox{iff} \quad AD \nmDelta B. 
	\end{equation}
\end{description}
Syntax splitting is the combination of \relevance\ and \indep:
\begin{description}
	\item[\textbf{\splitpost}]
	An inductive inference operator satisfies \textbf{\splitpost} \cite{KernIsbernerBeierleBrewka2020KR} if it satisfies \relevance\ and \indep.
\end{description}

Among the inductive inference operators investigated in \cite{KernIsbernerBeierleBrewka2020KR}, only reasoning with c-representations satisfies \splitpost.
\section{System W} %
\label{sec:systemW_preferred_structure}

Recently, system W has been introduced as a new inductive inference operator \cite{KomoBeierle2022AMAI}.
	System W takes into account both the tolerance information expressed by the 
	ordered partition of \(\Delta\) and the structural information which conditionals are falsified.
	
\begin{definition}[inclusion maximal tolerance partition \cite{Pearl1990systemZTARK}]
	A conditional \(\satzCL{B}{A}\) is \emph{tolerated} by \(\Delta\) if there exists a world \(\omega \in \Omega_\Sigma\) such that \(\omega\) verifies \(\satzCL{B}{A}\) and \(\omega\) does not falsify any conditional in \(\Delta\).
	The inclusion maximal \emph{tolerance partition} \(\fctOP{\Delta} = (\Delta^0, \dots, \Delta^k)\) of a consistent belief base \(\Delta\) is defined as follows.
	The first set \(\Delta^0\) in the tolerance partitioning contains all conditionals from \(\Delta\) that are tolerated by \(\Delta\).
	Analogously, \(\Delta^i\) contains all conditionals from \(\Delta\backslash (\bigcup_{j < i} \Delta^{j})\) which are tolerated by \(\Delta\backslash (\bigcup_{j < i} \Delta^{j})\), until \(\Delta\backslash (\bigcup_{j < k+1} \Delta^{j}) = \emptyset\). %
\end{definition}

It is well-known that \(\fctOP{\Delta}\) exists iff  \(\Delta\) is consistent; moreover,
because the \(\Delta^i\) are chosen inclusion-maximal, the tolerance partitioning is unique \cite{Pearl1990systemZTARK}.

	\begin{definition}[$\falsW^j$, $\falsW$, preferred structure \worderDelta\ on worlds \cite{KomoBeierle2022AMAI}]
		\label{def_w_structure_worlds}
		Consider a consistent belief base \(\Delta = \{ r_i=(B_i|A_i) \mid i \in\{1, \ldots, n\} \}\)
		with the tolerance partition \(\fctOP{\Delta} = (\Delta^1, \dots, \Delta^k)\).
		For $j = 0,\dots ,k$, the functions $\falsW^j$ and $\falsW$
		are the \emph{functions mapping worlds to the set of falsified conditionals}
		from the set $\Delta^j$ in the tolerance partition and from \(\Delta\), respectively, given by
		\begin{align}
			\falsW^j (\omega) &:= \{ r_i \in \Delta^ j \mid \omega \models A_i \overline{B_i}  \},\label{eq_mapping_f_j} \\
			\label{eq_mapping_f}
			\falsW(\omega) &:=  \{ r_i \in \Delta \mid \omega \models A_i \overline{B_i }  \}.
		\end{align}
		The \emph{\namePS on worlds} is given by the binary relation
		${\worderDelta} \subseteq \Omega \times \Omega$ 
		defined by, for any $\omega, \omega' \in \Omega$,
		\begin{align}
			\nonumber
			\omega \worderDelta \omega'  \thickspace\thickspace \text{iff} \thickspace\thickspace &\textrm{there exists $m \in \{0\, , \ldots \, , k \}$ such that }\\
			\nonumber
			&\falsW^i(\omega)  = \falsW^i(\omega') \quad \forall i  \in  \{  m + 1 \, , \ldots \, , k  \}, \,\,\textrm{and}\\
			\label{eq_sz_structure}
			& \falsW^m(\omega) \subsetneqq  \falsW^m(\omega') \, . 
		\end{align}
	\end{definition}
	
	Thus, $\omega \worderDelta \omega'$ if and only if $\omega $ falsifies strictly less conditionals than $\omega'$ in the partition with the biggest index $m$  where the conditionals falsified by $\omega $ and $\omega'$ differ. 
Note, that \worderDelta\ is a strict partial order.
	The inductive inference operator system W based on \worderDelta\ is defined as follows.
	
	\begin{definition}[system W, \wableitDelta \cite{KomoBeierle2022AMAI}]
		\label{def_sz_inference}
		Let \(\Delta\) be a belief base and \(A, B\) be formulas.
		Then $B$ is a  \emph{\systemSZ inference from $A$ (in the context of \(\Delta\))}, denoted \(A \wableitDelta B \) if
		for every \(\omega' \in \Omega_{A\overline{B}}\) there is an \(\omega \in \Omega_{A B}\) such that \(\omega \worderDelta \omega '\).
	\end{definition}

        System W extends system Z and c-inference and enjoys further desirable properties for nonmonotonic reasoning like avoiding the drowning problem. For more information on system W we refer to \cite{KomoBeierle2022AMAI}.
We         
illustrate system W with an example.

\begin{example}
	\label{ex:synsplit_different_inf_op}
	Consider the belief base \(\Delta = \{\satzCL{f}{b}, \satzCL{\ol{v}}{d}, \allowbreak \satzCL{b}{p}, \satzCL{\ol{f}}{p}\}\) over the signature \(\Sigma = \{b, p, f, v, d\}\)
        from \cite[Example~2]{KernIsbernerBeierleBrewka2020KR}
          with the intended meanings
          birds ($b$),
          penguins ($p$),
          being able to fly ($f$),
  being visible in the night ($v$),
          dark objects ($d$). 
The preferred structure \worderDelta\ on worlds is given in in Figure~\ref{fig_example_systemw}.

	We have $\Delta = \Delta_1 \splitcup \Delta_2$ with $\Sigma_1 = \{b, p, f\}\), \(\Sigma_2 = \{v, d \}$ and 
	$\Delta_1 = \{(f|b), (b|p), (\ol{f} | p)\}$,  
	$\Delta_2 = \{(\ol{v}|d)\}$. 

	The conditional \(\satzCL{\ol{v}}{d}\) can be deduced from \(\Delta\) with 
	every inductive inference operator
	because of (DI).
	But the conditional \(\satzCL{\ol{v}}{dp}\) cannot be deduced from \(\Delta\) with either p-entailment and System Z;
	in both cases, the additional information \(p\) from an independent part of the signature prevents the deduction of \(\neg v\).
	Therefore, p-entailment and system Z do not fulfil \splitpost.
        Using the preferred structure \worderDelta\
        given in Figure~\ref{fig_example_systemw},
        it is straightforward to verify that for each world \(\omega'\) with \(\omega' \models dpv\) there is a world \(\omega\) with \(\omega \models dp\ol{v}\) such that \(\omega \worderDelta \omega'\). Thus, system W licences the inference \(dp \wableitDelta \ol{v}\),
        complying with \splitpost.
\end{example}

  \begin{figure}%
\renewcommand{\ALvariableE}{\ensuremath{\mathit{b}}}
\renewcommand{\ALvariableZ}{\ensuremath{\mathit{p}}}
\renewcommand{\ALvariableD}{\ensuremath{\mathit{f}}}
\renewcommand{\ALvariableV}{\ensuremath{\mathit{v}}}
\renewcommand{\ALvariableF}{\ensuremath{\mathit{d}}}
\renewcommand{\druckeWeltFuenf}[5]{\ensuremath{\mathit{#1\setzeAbstandZwischenLiteralen\hspace{-0.5pt}#2\setzeAbstandZwischenLiteralen\hspace{-1pt}#3\setzeAbstandZwischenLiteralen\hspace{-1pt}#4\setzeAbstandZwischenLiteralen\hspace{-1pt}#5}}} %
  \centering
  \footnotesize%
  \begin{tikzpicture}
  	\tikzstyle{world}=[inner sep=.5mm,]
  	\tikzstyle{conn} = [->]
  	\tikzstyle{conn1} = []
  	
  	\def\xdiff{1.4}
  	\def\txdiff{.9}
  	\def\ydiff{0.85}
  	
  	\node (11) at (0,0) [world] {\omegaFuenf{17}};
  	\node (21) at (-\xdiff,-\ydiff) [world] {\omegaFuenf{20}};
  	\node (22) at (0,-.8*\ydiff) [world] {\omegaFuenf{18}};
  	\node (23) at (\xdiff,-\ydiff) [world] {\omegaFuenf{19}};
  	\node (31) at (-1.5*\xdiff,-2*\ydiff) [world] {\omegaFuenf{01}};
  	\node (32) at (1.5*\xdiff,-2*\ydiff) [world] {\omegaFuenf{21}};
  	\node (41) at (-2.5*\xdiff,-3*\ydiff) [world] {\omegaFuenf{02}};
  	\node (42) at (-1.5*\xdiff,-2.9*\ydiff) [world] {\omegaFuenf{03}};
  	\node (43) at (-.5*\xdiff,-2.8*\ydiff) [world] {\omegaFuenf{04}};
  	\node (44) at (.5*\xdiff,-2.8*\ydiff) [world] {\omegaFuenf{22}};
  	\node (45) at (1.5*\xdiff,-2.9*\ydiff) [world] {\omegaFuenf{23}};
  	\node (46) at (2.5*\xdiff,-3*\ydiff) [world] {\omegaFuenf{24}};
  	\node (51) at (-.75*\xdiff,-3.9*\ydiff) [world] {\omegaFuenf{05}};
  	\node (52) at (.75*\xdiff,-3.9*\ydiff) [world] {\omegaFuenf{13}};
  	\node (61) at (-4*\txdiff,-4.9*\ydiff) [world] {\omegaFuenf{14}};
  	\node (62) at (-3*\txdiff,-5.1*\ydiff) [world] {\omegaFuenf{09}};
  	\node (63) at (-2*\txdiff,-5.2*\ydiff) [world] {\omegaFuenf{29}};
  	\node (64) at (-1*\txdiff,-5.2*\ydiff) [world] {\omegaFuenf{15}};
  	\node (65) at (0*\txdiff,-5.2*\ydiff) [world] {\omegaFuenf{08}};
  	\node (66) at (1*\txdiff,-5.2*\ydiff) [world] {\omegaFuenf{16}};
  	\node (67) at (2*\txdiff,-5.2*\ydiff) [world] {\omegaFuenf{06}};
  	\node (68) at (3*\txdiff,-5.1*\ydiff) [world] {\omegaFuenf{25}};
  	\node (69) at (4*\txdiff,-4.9*\ydiff) [world] {\omegaFuenf{07}};
  	\node (71) at (-4*\txdiff,-7*\ydiff) [world] {\omegaFuenf{26}};
  	\node (72) at (-3*\txdiff,-7*\ydiff) [world] {\omegaFuenf{10}};
  	\node (73) at (-2*\txdiff,-7*\ydiff) [world] {\omegaFuenf{30}};
  	\node (74) at (-1*\txdiff,-7*\ydiff) [world] {\omegaFuenf{11}};
  	\node (75) at (0*\txdiff,-7*\ydiff) [world] {\omegaFuenf{12}};
  	\node (76) at (1*\txdiff,-7*\ydiff) [world] {\omegaFuenf{27}};
  	\node (77) at (2*\txdiff,-7*\ydiff) [world] {\omegaFuenf{28}};
  	\node (78) at (3*\txdiff,-7*\ydiff) [world] {\omegaFuenf{31}};
  	\node (79) at (4*\txdiff,-7*\ydiff) [world] {\omegaFuenf{32}};

  	\draw[conn] (21) -- (11);
  	\draw[conn] (22) -- (11);
  	\draw[conn] (23) -- (11);
  	\draw[conn] (31) -- (21);
  	\draw[conn] (32) -- (21);
  	\draw[conn] (31) -- (22);
  	\draw[conn] (32) -- (22);
  	\draw[conn] (31) -- (23);
  	\draw[conn] (32) -- (23);
  	\draw[conn] (41) -- (31);
  	\draw[conn] (42) -- (31);
  	\draw[conn] (43) -- (31);
  	\draw[conn] (44) -- (32);
  	\draw[conn] (45) -- (32);
  	\draw[conn] (46) -- (32);
  	\draw[conn] (51) -- (41);
  	\draw[conn] (51) -- (42);
  	\draw[conn] (51) -- (43);
  	\draw[conn] (51) -- (44);
  	\draw[conn] (51) -- (45);
  	\draw[conn] (51) -- (46);
  	\draw[conn] (52) -- (41);
  	\draw[conn] (52) -- (42);
  	\draw[conn] (52) -- (43);
  	\draw[conn] (52) -- (44);
  	\draw[conn] (52) -- (45);
  	\draw[conn] (52) -- (46);
  	\draw[conn] (61) -- (51);
  	\draw[conn] (62) -- (51);
  	\draw[conn] (63) -- (51);
  	\draw[conn] (64) -- (51);
  	\draw[conn] (65) -- (51);
  	\draw[conn] (66) -- (51);
  	\draw[conn] (67) -- (51);
  	\draw[conn] (68) -- (51);
  	\draw[conn] (69) -- (51);
  	\draw[conn] (61) -- (52);
  	\draw[conn] (62) -- (52);
  	\draw[conn] (63) -- (52);
  	\draw[conn] (64) -- (52);
  	\draw[conn] (65) -- (52);
  	\draw[conn] (66) -- (52);
  	\draw[conn] (67) -- (52);
  	\draw[conn] (68) -- (52);
  	\draw[conn] (69) -- (52);
  	\node(67center) at (0,-6.1*\ydiff) [circle,draw] {};
  	\draw[conn1] (71) -- (67center);
  	\draw[conn1] (72) -- (67center);
  	\draw[conn1] (73) -- (67center);
  	\draw[conn1] (74) -- (67center);
  	\draw[conn1] (75) -- (67center);
  	\draw[conn1] (76) -- (67center);
  	\draw[conn1] (77) -- (67center);
  	\draw[conn1] (78) -- (67center);
  	\draw[conn1] (79) -- (67center);
  	\draw[conn] (67center) to [bend left=15] (61);
  	\draw[conn] (67center) to [bend left=10] (62);
  	\draw[conn] (67center) to [bend left=10] (63);
  	\draw[conn] (67center) to [bend left=10] (64);
  	\draw[conn] (67center) -- (65);
  	\draw[conn] (67center) to [bend right=10] (66);
  	\draw[conn] (67center) to [bend right=10] (67);
  	\draw[conn] (67center) to [bend right=10] (68);
  	\draw[conn] (67center) to [bend right=15] (69);
  \end{tikzpicture}
\caption{The preferred structure on worlds induced by the belief base \(\Delta\) from Example~\ref{ex:synsplit_different_inf_op}. Edges that can be obtained from transitivity are omitted for lucidity.}
\label{fig_example_systemw}
\end{figure}
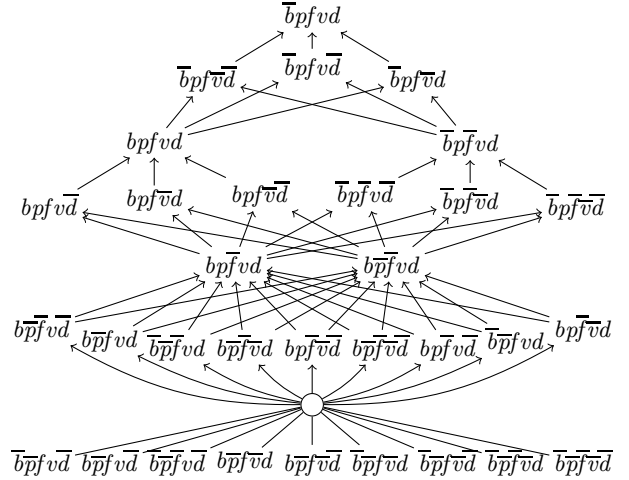

While 
Example~\ref{ex:synsplit_different_inf_op} is only an example of syntax splitting with system W, we will show that system W satisfies \splitpost\
and thus respects all syntax splittings.
\section{System W fulfils Syntax Splitting}
\label{sec:systemW_splitting}

In this section, we will evaluate system W with respect to the syntax splitting postulates.

For proving that system W fulfils syntax splitting, we first present four lemmas on the properties of \worderDelta\ in the presence of a syntax splitting \(\Delta = \Delta_1 \splitcup \Delta_2\).
Note, that we consider the belief bases \(\Delta_1, \Delta_2\) as belief bases over the signature \(\Sigma = \Sigma_1 \cup \Sigma_2\) in this section.
Thus, in particular
  \(\worder{\Delta_1}\) and \(\worder{\Delta_2}\) are relations on \(\Omega_\Sigma\) and the inference relations induced by \(\Delta_1, \Delta_2\) are calculated with respect to \(\Sigma\).

The following Lemma~\ref{lem:tolerance_partition_splitting} shows how a syntax splitting on a belief base carries over to the corresponding inclusion maximal tolerance partitioning.
\begin{lemma}
	\label{lem:tolerance_partition_splitting}
	Let \(\Delta = \Delta_1 \splitcup \Delta_2\) be a consistent belief base with syntax splitting.
	Let  \(\fctOP{\Delta} = (\Delta^0, \dots, \Delta^k)\) be the inclusion maximal tolerance partitioning of \(\Delta\).
	Let \(\fctOP{\Delta_i} = (\Delta^0_i, \dots, \Delta^{l_i}_i)\) be the inclusion maximal tolerance partition of \(\Delta_i\) for \(i=1, 2\).
	\begin{enumerate}
		\item For \(i=1,2\) and \(j=0, \dots, l_i\) we have \(\Delta^j_i = \Delta^j \cap \Delta_i\) and thus especially \(\Delta^j_i \subseteq \Delta^j\). 
		\item \(\max \{l_1, l_2\} = k\)
		\item If \(l_1 \leq l_2\), then
		\(
		\Delta^j = \begin{cases}
			\Delta_1^j \cup \Delta_2^j & \text{for } j =1, \dots, l_1 \\
			\Delta_2^j & \text{for } j = l_1+1, \dots, k.
		\end{cases}
		\)
	\end{enumerate}
\end{lemma}

If we have \(\omega \worderDelta \omega'\), then there is some conditional \(r\) that falsifies \(\omega'\) but not \(\omega\) and thus causes the \(\subsetneqq\) relation in \eqref{eq_sz_structure} in Definition~\ref{def_w_structure_worlds}. If \(\Delta = \Delta_1 \splitcup \Delta_2\), this \(r\) is either in \(\Delta_1\) or in \(\Delta_2\). 
Lemma~\ref{lem:worder_split} states that the relation \(\omega \worder{\Delta} \omega'\) can also be obtained using only \(\Delta_1\) or only \(\Delta_2\).

\begin{lemma}
	\label{lem:worder_split}
	Let \(\Delta = \Delta_1 \splitcup \Delta_2\) %
	and
	let \(\omega, \omega' \in \Omega\).
	If \(\omega \worderDelta \omega'\), then \(\omega \worder{\Delta_1} \omega'\) or \(\omega \worder{\Delta_2} \omega'\). %
\end{lemma}

Note, that both \(\omega \worder{\Delta_1} \omega'\) and \(\omega \worder{\Delta_2} \omega'\) might be true.

The next Lemma~\ref{lem:worder_merge} considers the reverse direction of Lemma~\ref{lem:worder_split} and shows a situation where we can infer \(\omega \worderDelta \omega'\) from \(\omega \worder{\Delta_1} \omega'\) for a belief base with syntax splitting.

\begin{lemma}
	\label{lem:worder_merge}
	Let \(\Delta = \Delta_1 \splitcup \Delta_2\) %
	and	let \(\omega, \omega' \in \Omega\).
	If \(\omega \worder{\Delta_1} \omega'\) and \(\worldMargin{\omega}{\Sigma_2} = \worldMargin{\omega'}{\Sigma_2}\), then \(\omega \worderDelta \omega'\).
\end{lemma}

The next Lemma~\ref{lem:assignment_irrelevant} captures that in a world the variable assignment for variables that do not occur in the belief set has no influence on the position of this world in the resulting preferential structure on worlds.

\begin{lemma}
	\label{lem:assignment_irrelevant}
	Let \(\Delta = \Delta_1 \splitcup \Delta_2\) %
	and \(\omega^a, \omega^b, \omega' \in \Omega_\Sigma\) with \(\worldMargin{\omega^a}{\Sigma_1} = \worldMargin{\omega^b}{\Sigma_1}\).
	Then we have \(\omega^a \worder{\Delta_1} \omega'\) iff \(\omega^b \worder{\Delta_1} \omega'\).
\end{lemma}

Now we can show that system W fulfils \relevance\ and \indep.

\begin{proposition}
	\label{prop:systemW_rel}
	System W fulfils \relevance.
\end{proposition}
\begin{proof}
	Let \(\Delta = \Delta_1 \splitcup \Delta_2\) %
	and
	let \(A, B \in \LpropSig{\Sigma_1}\) be propositional formulas.
	W.l.o.g. we need to show that 
	\begin{equation}
		\label{eq:system_w_Rel_proof}
	A \wableitDelta B \text{ if and only if } A \wableit{\Delta_1} B.
	\end{equation}
	\emph{Direction \(\Rightarrow\) of \eqref{eq:system_w_Rel_proof}:} \quad
	Assume that \(A \wableitDelta B\). We need to show that \(A \wableit{\Delta_1} B\).
	Let \(\omega'\) be any world in \(\Omega_{A\ol{B}}\).
	Now choose \(\omega_{\mathit{min}}' \in \Omega\) such that 
	\begin{enumerate}
		\item \(\omega_{\mathit{min}}' \wordereqDelta \omega' \),
		\item \(\worldMargin{\omega'}{\Sigma_1} = \worldMargin{\omega_{\mathit{min}}'}{\Sigma_1}\), and
		\item there is no world \(\omega_{\mathit{min}2}'\) with \(\omega_{\mathit{min}2}' < \omega_{\mathit{min}}'\) that fulfils (1.) and (2.).
	\end{enumerate}
	Such an \(\omega_{\mathit{min}}'\) exists because \(\omega'\) fulfils properties (1.)\ and (2.), \(\worderDelta\) is irreflexive and transitive, and there are only finitely many worlds in \(\Omega\).
	
	Because of (2.)\ and because \(\omega' \models A\ol{B}\) we have that  \(\omega_{\mathit{min}}' \models A\ol{B}\).
	Because \(A \wableitDelta B\), there is a world \(\omega\) such that \(\omega \models AB\) and \(\omega \worderDelta \omega_{\mathit{min}}'\).
	Lemma~\ref{lem:worder_split} yields that \(\omega \worder{\Delta_1} \omega_{\mathit{min}}'\) or \(\omega \worder{\Delta_2} \omega_{\mathit{min}}'\).
	
	The case \(\omega \worder{\Delta_2} \omega_{\mathit{min}}'\) is not possible: Assuming \(\omega \worder{\Delta_2} \omega_{\mathit{min}}'\), it follows that \(\omega_{\mathit{min}2}' = (\worldMargin{\omega_{\mathit{min}}'}{\Sigma_1} \cdot \worldMargin{\omega}{\Sigma_2}) \worder{\Delta_2} \omega_{\mathit{min}}'\) with Lemma~\ref{lem:assignment_irrelevant}. %
	With Lemma~\ref{lem:worder_merge} it follows that \(\omega_{\mathit{min}2}' \worderDelta \omega_{\mathit{min}}'\). This contradicts (3.).
	Hence, \(\omega \worder{\Delta_1} \omega_{\mathit{min}}'\). %
	Because of (2.)\ and Lemma~\ref{lem:assignment_irrelevant}
	it follows that \(\omega \worder{\Delta_1} \omega'\). %
	
	As we can find an \(\omega\) such that \(\omega \worder{\Delta_1} \omega'\) and \(\omega \models AB\) for every \(\omega' \models A\ol{B}\) we have that \(A \wableit{\Delta_1} B\).

	\emph{Direction \(\Leftarrow\) of \eqref{eq:system_w_Rel_proof}:} \quad
	Assume that \(A \wableit{\Delta_1} B\). We need to show that \(A \wableitDelta B\).
	Let \(\omega'\) be any world in \(\Omega_{A\ol{B}}\).
	Because \(A \wableit{\Delta_1} B\), there is a world \(\omega^*\) such that \(\omega^* \models AB\) and \(\omega^* \worder{\Delta_1} \omega'\).
	Let \(\omega = (\worldMargin{\omega^*}{\Sigma_1} \cdot \worldMargin{\omega'}{\Sigma_2})\).
	Because \(\omega^* \models AB\) we have that \(\omega \models AB\).
	Furthermore, with Lemma~\ref{lem:assignment_irrelevant} it follows that \(\omega \worder{\Delta_1} \omega'\) and thus \(\omega \worderDelta \omega'\) with Lemma~\ref{lem:worder_merge}.
	
	As we can construct \(\omega\) such that \(\omega \worderDelta \omega'\) and \(\omega \models AB\) for every \(\omega' \models A\ol{B}\) we have that \(A \worderDelta B\).
\end{proof}

\begin{proposition}
	\label{prop:systemW_ind}
	System W fulfils \indep.
\end{proposition}
\begin{proof}
	Let \(\Delta = \Delta_1 \splitcup \Delta_2\). %
	W.l.o.g. let \(A, B \in \LpropSig{\Sigma_1}\) and \(C \in \LpropSig{\Sigma_2}\) be propositional formulas such that \(C\) is consistent.
	We need to show that
	\begin{equation}
		\label{eq:system_w_Ind_proof}
	A \wableitDelta B \text{ if and only if } AC \wableitDelta BC.
	\end{equation}

	\emph{Direction \(\Rightarrow\) of \eqref{eq:system_w_Ind_proof}:} \quad
	Assume that \(A \wableitDelta B\). We need to show that \(AC \wableitDelta B\).
	Let \(\omega'\) be any world in \(\Omega_{A\ol{B}C}\).
	Now choose \(\omega_{\mathit{min}}' \in \Omega\) such that 
	\begin{enumerate}
		\item \(\omega_{\mathit{min}}' \wordereqDelta \omega' \),
		\item \(\worldMargin{\omega'}{\Sigma_1} = \worldMargin{\omega_{\mathit{min}}'}{\Sigma_1}\), and
		\item there is no world \(\omega_{\mathit{min}2}'\) with \(\omega_{\mathit{min}2}' < \omega_{\mathit{min}}'\) that fulfils (1.)\ and (2.).
	\end{enumerate}
	Such an \(\omega_{\mathit{min}}'\) exists because \(\omega'\) fulfils properties (1.)\ and (2.), \(\worderDelta\) is irreflexive and transitive, and there are only finitely many worlds in \(\Omega\).
	Because of (2.)\ and because \(\omega' \models A\ol{B}C\) we have that  \(\omega_{\mathit{min}}' \models A\ol{B}\).
	Because \(A \wableitDelta B\), there is a world \(\omega^*\) such that \(\omega^* \models AB\) and \(\omega^* \worderDelta \omega_{\mathit{min}}'\).
	Lemma~\ref{lem:worder_split} yields that either \(\omega^* \worder{\Delta_1} \omega_{\mathit{min}}'\) or \(\omega^* \worder{\Delta_2} \omega_{\mathit{min}}'\).
	
	The case \(\omega^* \worder{\Delta_2} \omega_{\mathit{min}}'\) is not possible: Assuming \(\omega^* \worder{\Delta_2} \omega_{\mathit{min}}'\), it follows that \(\omega_{\mathit{min}2}' = (\worldMargin{\omega_{\mathit{min}}'}{\Sigma_1} \cdot \worldMargin{\omega^*}{\Sigma_2}) \worder{\Delta_2} \omega_{\mathit{min}}'\) with Lemma~\ref{lem:assignment_irrelevant}.
	With Lemma~\ref{lem:worder_merge} it follows that \(\omega_{\mathit{min}2}' %
	\worderDelta \omega_{\mathit{min}}'\). 
	This contradicts (3.).
	Hence, \(\omega^* \worder{\Delta_1} \omega_{\mathit{min}}'\).
	Let \(\omega = (\worldMargin{\omega^*}{\Sigma_1} \cdot \worldMargin{\omega'}{\Sigma_2})\).
	Because \(\omega^* \models AB\) we have that \(\omega \models AB\). Because \(\omega' \models C\) we have that \(\omega \models C\).
	Because of (2.)\ and Lemma~\ref{lem:assignment_irrelevant} it follows that
	\(\omega \worder{\Delta_1} \omega'\) and thus with Lemma~\ref{lem:worder_merge} \(\omega \worderDelta \omega'\).
	
	As we can construct an \(\omega\) such that \(\omega \worderDelta \omega'\) and \(\omega \models ABC\) for every \(\omega' \models A\ol{B}C\) we have that \(AC \wableitDelta B\).

	\emph{Direction \(\Leftarrow\) of \eqref{eq:system_w_Ind_proof}:} \quad
	Assume that \(AC \wableitDelta BC\). We need to show that \(A \wableitDelta B\).
	Let \(\omega'\) be any world in \(\Omega_{A\ol{B}}\).
	Now choose \(\omega_{\mathit{min}}' \in \Omega\) such that 
	\begin{enumerate}
		\item \(\omega_{\mathit{min}}' \models C\)
		\item \(\worldMargin{\omega'}{\Sigma_1} = \worldMargin{\omega_{\mathit{min}}'}{\Sigma_1}\), and
		\item there is no world \(\omega_{\mathit{min}2}'\) with \(\omega_{\mathit{min}2}' < \omega_{\mathit{min}}'\) that fulfils (1.) and (2.).
	\end{enumerate}
	Such an \(\omega_{\mathit{min}}'\) exists because \(C\) is consistent, \(\worderDelta\) is irreflexive and transitive, and there are only finitely many worlds in \(\Omega\).
	Because of (2.)\ and because \(\omega' \models A\ol{B}\) we have that  \(\omega_{\mathit{min}}' \models A\ol{B}\).
	Because of (1.)\ we have that \(\omega_{\mathit{min}}' \models C\).
	Because \(AC \wableitDelta B\), there is a world \(\omega^*\) such that \(\omega^* \models ABC\) and \(\omega^* \worderDelta \omega_{\mathit{min}}'\).
	Lemma~\ref{lem:worder_split} yields that either \(\omega^* \worder{\Delta_1} \omega_{\mathit{min}}'\) or \(\omega^* \worder{\Delta_2} \omega_{\mathit{min}}'\).
	
	The case \(\omega^* \worder{\Delta_2} \omega_{\mathit{min}}'\) is not possible: Assuming \(\omega^* \worder{\Delta_2} \omega_{\mathit{min}}'\), it follows that \(\omega_{\mathit{min}2}' = (\worldMargin{\omega_{\mathit{min}}'}{\Sigma_1} \cdot \worldMargin{\omega^*}{\Sigma_2}) \worder{\Delta_2} \omega_{\mathit{min}}'\) with Lemma~\ref{lem:assignment_irrelevant}.
	With Lemma~\ref{lem:worder_merge} it follows that \(\omega_{\mathit{min}2}' %
	\worderDelta \omega_{\mathit{min}}'\).
	This contradicts (3.).
	Hence, \(\omega^* \worder{\Delta_1} \omega_{\mathit{min}}'\).
	Let \(\omega = (\worldMargin{\omega^*}{\Sigma_1} \cdot \worldMargin{\omega'}{\Sigma_2})\).
	Because \(\omega^* \models AB\) we have that \(\omega \models AB\).
	Because of (2.)\ and Lemma~\ref{lem:assignment_irrelevant} it follows that
	\(\omega \worder{\Delta_1} \omega'\) and thus with Lemma~\ref{lem:worder_merge} \(\omega \worderDelta \omega'\).
	
	As we can construct an \(\omega\) such that \(\omega \worderDelta \omega'\) and \(\omega \models AB\) for every \(\omega' \models A\ol{B}\) we have that \(A \wableitDelta B\).
\end{proof}

Combining Propositions~\ref{prop:systemW_rel} and \ref{prop:systemW_ind} yields that system W fulfils \splitpost.
\begin{proposition}
	System W fulfils \splitpost.
\end{proposition}

\section{Conclusions and Further Work}
\label{sec:conclusion}

In this short paper, we showed that the recently introduced System W that extends rational closure and c-inference, also fully complies with syntax splitting.
In our current work, we are studying the effect of syntax splitting on the preferred structure on worlds in more detail, and are investigating further properties of system W.

\bibliographystyle{kr}
\bibliography{literature}

\end{document}